%% file: root.tex
\documentclass[letterpaper, 10 pt, conference]{ieeeconf} %

\IEEEoverridecommandlockouts                              %
\usepackage{amsmath,amssymb,amsfonts}

\newtheorem{problem}{Problem}
\newtheorem{lemma}{Lemma}
\usepackage{bm}
\usepackage{graphicx}
\usepackage{xcolor}
\usepackage{url}
\usepackage[colorlinks=true,linkcolor=black,citecolor=black,urlcolor=blue]{hyperref}
\usepackage{subfig}
\usepackage{cite}

\input{math_commands.tex}

\renewcommand{\baselinestretch}{.97}

\title{\LARGE \bf
Adapting to Changes in Agent Behavior via Finite-Depth Policy Sensitivity\vspace{-0.2cm}
}

\author{Lan Shi, Daigo Shishika, and Xuan Wang%
\thanks{The authors are with George Mason University, USA.}
}

\begin{document}

\maketitle
\thispagestyle{empty}
\pagestyle{empty}

\begin{abstract}

Adapting a reinforcement learning policy to changes in another agent's behavior typically requires a large amount of new interaction data. 
Policy sensitivity provides a first-order prediction of how a locally optimal policy changes with a behavioral parameter, but its computation requires second-order derivatives whose effects propagate across future interactions.
We develop a finite-depth framework to estimate this sensitivity by
approximating the policy Hessian and mixed derivative using information
from a reference environment. 
The method features an adjustable propagation depth which determines where derivative propagation along the trajectory is truncated.
We characterize the derivative contributions omitted by finite-depth propagation and derive truncation-error bounds for the approximated derivatives and resulting policy sensitivity. The bounds are nonincreasing with propagation depth and vanish at full-horizon propagation.
Using a belief-driven pursuit--evasion game as a validation scenario, 
the proposed method generally achieves lower derivative-estimation errors as the propagation depth increases and outperforms the baseline methods in both estimation accuracy and policy adaptation.
The sensitivity-based initialization improves zero-shot return over direct transfer, and also shows advantages for the subsequent fine-tuning in the target environment.
\end{abstract}

\vspace{-2pt}
\section{Introduction}
\vspace{-2pt}

Reinforcement-learning policies deployed in interactive environments often need to operate with agents exhibiting different behaviors~\cite{lin2020model}. For example, an autonomous vehicle may encounter drivers with different levels of aggressiveness, while a robotic agent may interact with teammates or opponents following different decision strategies. Because the behavior of an interactive agent may be unknown before deployment, a policy optimized for one behavior may not perform well against another. On the other hand, retraining a policy from scratch for every newly encountered agent can be costly.
One approach to efficient policy adaptation is sensitivity-based adaptation, which characterizes how a locally optimal policy changes with respect to a problem parameter. The resulting sensitivity can be used in a first-order prediction to adapt a previously learned policy to a nearby target environment without learning from scratch. However, computing this sensitivity generally requires higher-order derivatives of the expected payoff.

In our setting, rather than directly parameterizing physical dynamics or reward~\cite{633827,barreto2017successor}, the environmental variation is induced by a change in another decision-making agent's policy.
The behavioral parameter therefore affects the transition dynamics experienced by the learning agent through the interactive agent's decisions. Because this effect propagates through subsequent states and the coupled decisions of both agents, a change in the interactive agent can influence the entire future trajectory.
Consequently, estimating these
effects from full trajectories can require substantial data, while
propagating them through a transition model over the full horizon can
be computationally expensive.
This motivates the problem studied in this paper: how to efficiently estimate policy sensitivity to changes in an interactive agent's behavior while providing theoretical guarantees on the resulting approximation. Figure~\ref{fig:intro} illustrates the interaction structure. %
\begin{figure}[t]
    \centering
    \includegraphics[width=0.8\linewidth]{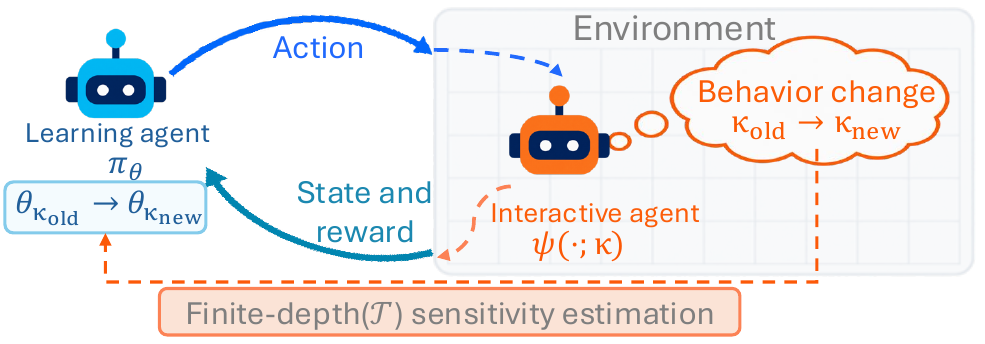}
    \caption{Adaptation to changes in an interactive agent's behavior by finite-depth policy sensitivity estimation.}
    \label{fig:intro}
    \vspace{-1.8em}
\end{figure}

\textit{{Literature review}}:
Existing approaches to policy adaptation under environmental variation include robust reinforcement learning and domain randomization, which train a common policy across a prescribed range of environments or disturbances \cite{pinto2017robust,calafiore2006scenario,zhang2020policy},
and meta-reinforcement learning, which learns an adaptation mechanism from a distribution of related tasks \cite{finn2017model,muthirayan2025meta}.
These approaches rely on training in multiple environments or task variations. In contrast, we consider a policy optimized for a single reference environment and how to adapt it to a particular change in the interactive agent.

Our formulation is closely related to sensitivity analysis and implicit differentiation for parameterized optimization \cite{zavala2009advanced,ratliff2016characterization}.
Differentiable optimization and bilevel-learning methods characterize how an optimizer varies with problem parameters by differentiating its optimality conditions \cite{agrawal2019differentiable,gould2016differentiating,gros2019data}.
Related ideas have also been studied in reinforcement learning.
For example, sensitivity arguments have been used to analyze policy-gradient methods \cite{633827,wu2022understanding}, while implicit differentiation has been applied in model-based reinforcement learning to differentiate Bellman-optimal value functions with respect to model parameters \cite{nikishin2022control}.
Differentiation-based ideas also appear in meta-gradient reinforcement learning, where meta-parameters for training are optimized according to their effect on subsequent performance \cite{xu2018meta,xu2020meta}.
In our setting, implicit differentiation characterizes how a locally optimal policy changes with the behavioral parameter of the interactive agent.

Computing this policy sensitivity requires both the policy Hessian and the mixed derivative between the learner's policy parameter and the interactive agent's behavioral parameter.
Prior work in reinforcement learning has primarily studied the computation or approximation of policy Hessians for policy optimization \cite{kordabad2022quasi}.
Policy Hessians can be expressed in terms of trajectory-likelihood derivatives and estimated from sampled trajectories~\cite{marbach2001simulation}. 
Higher-order score-function estimators, including DiCE-type methods, provide another trajectory-based route to estimating higher-order derivatives~\cite{farquhar2019loaded}.
The work~\cite{shen2019hessian} also shows that Hessian--vector products can be estimated from trajectories without explicitly forming the full matrix.
Alternatively, when a differentiable transition model is available, second-order derivatives can be obtained by differentiating through the model and propagating derivative information along the state-transition dynamics~\cite{gros2019data}.
These approaches provide two complementary routes: trajectory-based estimation
avoids full model propagation but can require many samples, whereas model-based differentiation propagates derivative information through the model but may require propagation over the full horizon. This motivates a finite-depth approximation that combines reference-environment sampling with local model-based derivative propagation. While temporal truncation has also been used in approximate dynamic programming to reduce planning-horizon computation~\cite{farahmand2016truncated}, our method retains the boundary value while truncating its parameter-derivative propagation.

 \textit{\textbf{Contributions}} of this paper include:
(i) \textit{Finite-depth sensitivity estimation.}
By combining state sampling with local state-transition propagation, we develop a Bellman-based derivative decomposition and a finite-depth approximation of the policy Hessian and mixed derivative. An adjustable derivative-propagation depth $\mathcal{T}$ controls how far parameter-derivative effects are propagated through future interactions.
(ii) \textit{Truncation-error guarantees.}
We characterize the derivative contributions omitted by finite-depth
propagation and derive error bounds for the approximated derivatives
and resulting policy sensitivity. The bounds quantify the
depth--accuracy tradeoff. They are nonincreasing with propagation depth,
and vanish at full-horizon propagation.
(iii) \textit{Validation.}
We evaluate the proposed framework in a representative example (a pursuit-evasion game with a behavior-parameterized defender) through numerical studies, including sensitivity-estimation accuracy and adaptation performance.

\section{Problem Formulation}\label{sec:problem}

\subsection{Parameterized Interactive-Agent MDP}\label{sub:setting}
We consider a family of environments with two interactive agents: a learning agent, denoted by $\mathsf L$, and an interactive agent, denoted by $\mathsf O$. The learning agent follows a policy $\pi_\theta(\cdot\mid s)$, parameterized by $\theta\in\mathbb{R}^{d_\theta}$, where $\theta$ contains the trainable policy parameters. 
The interactive agent follows a policy $\psi(\cdot;\kappa)$ characterized by a behavioral parameter $\kappa\in\mathbb{R}^{d_\kappa}$. Different values of $\kappa$ represent different behavioral characteristics of the interactive agent, such as aggressiveness or action preference. Changes in $\kappa$ alter the behavior of the interactive agent and, consequently, the transition dynamics experienced by the learning agent. Accordingly, each fixed value of $\kappa$ defines an environment in this family, and a change in $\kappa$ may require the learning agent to adapt its policy.

We model the environment experienced by the learning agent as a finite-horizon Markov decision process indexed by $\kappa$,
$
\mathcal{M}(\kappa)=\left(\mathcal{S}, \mathcal{A}_\mathsf L, P_{\kappa}, r, \rho_0, \gamma,  H\right),
$
where $\mathcal{S}$ and $\mathcal{A}_\mathsf L$ are the state space and the learning agent's action space, respectively; $P_\kappa(\cdot\mid s,a_{\mathsf L})$ is the one-step transition kernel experienced by the learning agent; $r:\mathcal{S}\times\mathcal{A}_\mathsf L\times\mathcal{S}\to\mathbb{R}$ is the transition reward, evaluated on $(s,a_{\mathsf{L}},s')$; $\rho_0$ is the initial-state distribution; $\gamma \in (0,1]$ is the discount factor; and $H$ is the maximum episode length. Across this family of environments, only the behavior of the interactive agent varies with $\kappa$.  
The initial-state distribution, reward function, physical dynamics, and termination rule are independent of $\kappa$. 

To relate $P_\kappa$ to the interactive agent's behavior, let $a_{\mathsf L}\in\mathcal A_{\mathsf L}$ and $a_{\mathsf O}\in\mathcal A_{\mathsf O}$ denote the agents' actions, with both action spaces finite. After the learning agent selects $a_{\mathsf L}$, an intermediate state $\tilde s\in\mathcal{S}$ is generated according to $\tilde{s}\sim P^{\mathsf L}(\cdot\mid s,a_{\mathsf L}).$ The interactive agent then selects $a_{\mathsf O}\sim\psi(\cdot\mid\tilde{s};\kappa),$ after which the next state is generated according to $s'\sim P^{\mathsf O}(\cdot\mid\tilde{s},a_{\mathsf O}).$ 
This gives

{\small
\begin{equation}
\begin{aligned}
P_\kappa(ds'\mid s,a_{\mathsf L})
=&
\int_{\mathcal S}
\sum_{a_{\mathsf O}\in\mathcal A_{\mathsf O}}
\psi(a_{\mathsf O}\mid\tilde{s};\kappa)
\\
&\cdot
P^{\mathsf O}(ds'\mid\tilde{s},a_{\mathsf O})
P^{\mathsf L}(d\tilde{s}\mid s,a_{\mathsf L}).
\end{aligned}
\label{eq:induced-transition}
\end{equation}
}

\noindent Here,~\eqref{eq:induced-transition} allows the interaction to be represented as a parameterized MDP from the learning agent's perspective. Importantly, however, $\kappa$ does not directly parameterize the physical transition dynamics. Instead, it changes the decision policy of the interactive agent, and the resulting kernel $P_\kappa$ is induced by marginalizing that agent's action. Consequently, a perturbation in $\kappa$ affects not only the immediate transition but also subsequent state distributions and the 
learner's future decisions through repeated interaction.

Let 
$\xi=(s_0,a_{\mathsf L,0},\tilde s_0,a_{\mathsf O,0},
s_1,\ldots,
a_{\mathsf L,H-1},\tilde s_{H-1},a_{\mathsf O,H-1},s_H)$
denote an interaction trajectory, and let $\mathbb P_{\theta,\kappa}$ denote its distribution induced by
$\rho_0$, $\pi_\theta$, $P^{\mathsf L}$, $\psi(\cdot;\kappa)$, and $P^{\mathsf O}$. Let $r_t=r(s_t,a_{\mathsf L,t},s_{t+1})$ denote the reward at stage $t$. Then,
\begin{equation}
\label{eq:objective}
J(\theta ; \kappa)=\mathbb E_{{{\xi\sim\mathbb P_{\theta,\kappa}}}}\left[\sum_{t=0}^{H-1} \gamma^t r_t\right],
\end{equation}

To study policy adaptation within the family $\mathcal M(\kappa)$, we assume that a locally optimal policy has been obtained in a reference environment $\mathcal M(\kappa_{\rm old})$ and that the following information is available.

\begin{assumption}[Available Information and Regularity]\label{as:info-available}
In the reference environment $\mathcal{M}(\kappa_{\mathrm{old}})$, we assume:\\
(i) access to the locally optimal policy parameter $\theta^\star(\kappa_{\mathrm{old}})$ and sampled interactive trajectories generated in the reference environment under $\pi_{\theta^\star(\kappa_{\mathrm{old}})}$;\\
(ii) $P_\kappa(\cdot\mid s,a_{\mathsf L})$ is %
differentiable in $\kappa$ %
in a neighborhood of $\kappa_{\mathrm{old}}$.
The transition kernel $P_{\kappa_{\mathrm{old}}}(\cdot\mid s,a_{\mathsf L})$
is available,
and 
$\left.
\nabla_\kappa P_\kappa(\cdot\mid s,a_{\mathsf L})
\right|_{\kappa=\kappa_{\mathrm{old}}}$ is computable;\\
(iii) $\pi_\theta(a_{\mathsf L}\mid s)$ is twice continuously
differentiable with respect to $\theta$ in a neighborhood of
$\theta^\star(\kappa_{\mathrm{old}})$, and its first- and second-order
derivatives can be evaluated at
$\theta=\theta^\star(\kappa_{\mathrm{old}})$.\\
(iv) $\nabla_\theta J(\theta;\kappa)$ is continuously differentiable in a neighborhood of $\bigl(\theta^\star(\kappa_{\mathrm{old}}),\kappa_{\mathrm{old}}\bigr)$. Moreover,
$\nabla_\theta J
\bigl(\theta^\star(\kappa_{\mathrm{old}});
\kappa_{\mathrm{old}}\bigr)
=0$,
$\nabla_{\theta\theta}^2 J
\bigl(\theta^\star(\kappa_{\mathrm{old}});
\kappa_{\mathrm{old}}\bigr)
\prec0$.\\
\end{assumption}

\smallskip
Here, (i) specifies the reference information available for adaptation, which can be obtained from training and operating the reference policy; (ii) specifies the available transition-model information and its local regularity, which is available when the interaction between the learning and interactive agents is known from system design or can be identified from data;
and (iii)-(iv) are standard local differentiability and optimality conditions~\cite{gould2016differentiating}. %

\subsection{Sensitivity-based Adaptation}\label{sub:application}
Consider a change in the behavioral parameter: $\kappa_{\mathrm{new}}=\kappa_{\mathrm{old}}+\Delta\kappa$, %
where $\kappa_{\mathrm{new}}$ and $\kappa_{\mathrm{old}}$ are known or can be identified~\cite{9147924}. We refer to $\mathcal{M}(\kappa_{\mathrm{new}})$ as the
\emph{target environment}.
The reference policy parameter $\theta^\star(\kappa_{\mathrm{old}})$ locally optimal for $\mathcal M(\kappa_{\mathrm{old}})$ may no longer be locally optimal in the target environment. We use reference-environment information to construct an initialization for adaptation in the target environment. %
This is based on local sensitivity analysis to characterize how the locally optimal policy parameter $\theta^\star(\kappa)$ changes as $\kappa$ varies. We call this $\partial\theta^\star(\kappa)/\partial\kappa$ the policy-parameter sensitivity (policy sensitivity for short). 

Under Assumption~\ref{as:info-available}, the implicit function theorem
\cite[Theorem~9.28]{rudin1976principles} gives a locally unique,
continuously differentiable stationary branch
$\theta^\star(\kappa)$ %
satisfying %
\begin{equation}\label{eq:stationarity}
    \nabla_\theta J
    \left(\theta^\star(\kappa);\kappa\right)=0.
\end{equation}
Differentiating \eqref{eq:stationarity} with respect to $\kappa$ and evaluating at $\kappa=\kappa_{\mathrm{old}}$ yields
\begin{equation}\label{eq:sensitivity}
\mathbf H_{\mathrm{old}}\mathbf{S}_{\mathrm {old }}+\mathbf M_{\mathrm{old}}=0
\end{equation}
where the Hessian $\mathbf H_{\mathrm{old}} :=\nabla_{\theta\theta}^2J\bigl(\theta^\star(\kappa_{\mathrm{old}});\kappa_{\mathrm{old}}\bigr)\in\mathbb{R}^{d_\theta \times d_\theta}$ 
characterizes the local curvature of the learner's objective.
The mixed derivative $\mathbf M_{\mathrm{old}}:=\nabla_{\theta\kappa}^2J\bigl(\theta^\star(\kappa_{\mathrm{old}});\kappa_{\mathrm{old}}\bigr)\in\mathbb{R}^{d_\theta \times d_\kappa}$ characterizes how the learner's policy gradient changes with the behavioral parameter of the interactive agent.
The policy sensitivity is $\mathbf{S}_{\mathrm {old }}:=\partial \theta^\star(\kappa) /\left.\partial \kappa\right|_{\kappa=\kappa_\mathrm {old}}\in\mathbb{R}^{d_\theta \times d_\kappa}$.

For sufficiently small $\|\Delta\kappa\|_2$, differentiability of the local solution branch gives
{\small
\begin{equation}\label{eq_init}
\theta^\star\left(\kappa_{\mathrm {new }}\right)\approx \theta^\star\left(\kappa_{\mathrm {old }}\right) +%
\mathbf{S}_{\mathrm {old }}\Delta \kappa.%
\end{equation}
}Thus, the sensitivity at $\kappa_{\mathrm{old}}$ determines the first-order correction to the reference policy parameter under the change in the interactive agent. Based on this, \eqref{eq_init} can be used as a natural initialization for adaptation in $\mathcal M(\kappa_{\mathrm{new}})$. 
The remaining problem is therefore to estimate $\mathbf{S}_{\mathrm {old }}$, or equivalently the quantities $\mathbf H_{\mathrm{old}}$ and $\mathbf M_{\mathrm{old}}$ in \eqref{eq:sensitivity}, using the reference-environment information available in Section~\ref{sub:setting}.

\subsection{Problem of Interest}
\begin{problem}[Computing Policy Sensitivity for Adaptation]\label{prob}
    Given the reference policy parameter $\theta^\star(\kappa_\mathrm{old})$ and the information specified in Assumption~\ref{as:info-available}, \\
    (i) construct approximations $\widehat{\mathbf H}_{\mathrm{old}}$ and $\widehat{\mathbf M}_{\mathrm{old}}$, and use them to estimate policy sensitivity
    $\mathbf S_{\mathrm{old}}=-\mathbf H_{\mathrm{old}}^{-1}\mathbf M_{\mathrm{old}}.$\\
    (ii) characterize the approximation errors in $\widehat{\mathbf H}_{\mathrm{old}}$, $\widehat{\mathbf M}_{\mathrm{old}}$, and the resulting sensitivity estimate $\widehat{\mathbf S}_{\mathrm{old}}$.

\end{problem}

\begin{remark}[Motivation for Finite-Depth Approximation]
The derivatives $\mathbf H_{\mathrm{old}}$ and $\mathbf M_{\mathrm{old}}$ can be estimated from full sampled trajectories or computed by propagating model-based derivative information over the full horizon. 
However, the former can require substantial trajectory data, whereas the latter incurs full-horizon model propagation. 
In Section~\ref{sec:method}, we introduce an approach that provides a middle ground by combining these two alternatives.
Our method combines reference-environment sampling with model-based derivative propagation truncated to an adjustable propagation depth $\mathcal T$.
The resulting truncation errors in $\mathbf H_{\mathrm{old}}$, $\mathbf M_{\mathrm{old}}$, and $\mathbf S_{\mathrm{old}}$ are analyzed and quantified in Section~\ref{sec:sensitivity-error}. \hfill$\square$
\end{remark}

\section{Method}\label{sec:method}
To approximate $\mathbf H_{\mathrm{old}}$ and $\mathbf M_{\mathrm{old}}$, we first derive an exact Bellman decomposition that separates derivative contributions within a chosen propagation depth from those carried by the remaining-horizon boundary terms. We then construct a finite-depth approximation by fixing the boundary value at stage $\mathcal T$ and truncating its parameter derivatives.

\subsection{Bellman Derivative Decomposition}
\label{sub:finite-depth}
Let $V_t^{\theta,\kappa}(s):=\mathbb E\bigl[\sum_{\ell=t}^{H-1}\gamma^{\ell-t}r_\ell \,\big|\, s_t=s\bigr]$ denote the value function at stage $t$, where the expectation is over trajectories generated by $\pi_\theta$ under $P_\kappa$. Let $Q_t^{\theta,\kappa}(s,a_{\mathsf L})$ denote the corresponding expected return conditional on $s_t=s$ and $a_{\mathsf L,t}=a_{\mathsf L}$. %
With the terminal condition $V_H^{\theta,\kappa}(s)\equiv0$, these values satisfy the following recursions for $t=0,\dots,H-1$:

\vspace{-1em}
\begin{subequations}\label{eq:method-bellman}
{\small\begin{align}\label{eq:method-bellman-a}
Q_t^{\theta, \kappa}\left(s, a_{\mathsf{L}}\right)
&= \int_{\mathcal S} P_\kappa\left(d s^{\prime} \mid s, a_{\mathsf{L}}\right) \left[r\left(s, a_{\mathsf{L}}, s^{\prime}\right)+\gamma V_{t+1}^{\theta, \kappa}\left(s^{\prime}\right)\right] \\ \label{eq:method-bellman-b}
V_t^{\theta, \kappa}(s)&= \sum_{a_{\mathsf{L}} \in \mathcal{A}_{\mathsf{L}}} \pi_\theta\left(a_{\mathsf{L}} \mid s\right) Q_t^{\theta, \kappa}\left(s, a_{\mathsf{L}}\right) .
\end{align}}
\end{subequations}
Since $P_\kappa$ is independent of $\theta$, one has
{\small
\begin{equation}\label{eq:Q-theta-derivative}
\nabla_\theta
Q_t^{\theta,\kappa}(s,a_{\mathsf L})
=
\gamma
\int_{\mathcal S}
P_\kappa(ds'\mid s,a_{\mathsf L})
\nabla_\theta
V_{t+1}^{\theta,\kappa}(s').
\end{equation}
}

Define the current-stage policy derivative contribution
{\small
\begin{equation}
\label{eq:method-local-gradient}
G_t^{\theta,\kappa}(s)
:=
\sum_{a_{\mathsf L}\in\mathcal A_{\mathsf L}}
\nabla_\theta\pi_\theta(a_{\mathsf L}\mid s)
Q_t^{\theta,\kappa}(s,a_{\mathsf L}).
\end{equation}
}
Differentiating \eqref{eq:method-bellman-b} with respect to $\theta$ and using \eqref{eq:Q-theta-derivative} gives\footnote{Here and in the following, we assume that differentiation can be interchanged with the integrals in the Bellman and
state-distribution recursions under standard regularity conditions.}
{\small\begin{equation}
\label{eq:method-value-derivative}
\begin{aligned}
\nabla_\theta& V_t^{\theta,\kappa}(s)
=G_t^{\theta,\kappa}(s)\\
&+\gamma
\sum_{a_{\mathsf L}\in\mathcal A_{\mathsf L}}
\pi_\theta(a_{\mathsf L}\mid s)
\int_{\mathcal S}
P_\kappa(ds'\mid s,a_{\mathsf L})
\nabla_\theta V_{t+1}^{\theta,\kappa}(s').
\end{aligned}
\end{equation}}

We also characterize the state distribution induced by the policy and transition dynamics. For a given initial state $s_0$, let
$\mu_t^{\theta,\kappa}(\cdot\mid s_0)$ denote the state distribution
at stage $t$ obtained by propagation through
$\pi_\theta$ and $P_\kappa$, recursively defined by

\vspace{-0.5em}
{\small\begin{subequations}
\begin{align}
\mu_0^{\theta,\kappa}(ds\mid s_0)
&:=
\delta_{s_0}(ds),
\\
\mu_{t+1}^{\theta,\kappa}(ds'\mid s_0)
&:=
\int_{\mathcal S}
\mu_t^{\theta,\kappa}(ds\mid s_0)
\sum_{a_{\mathsf L}\in\mathcal A_{\mathsf L}}
\pi_\theta(a_{\mathsf L}\mid s)
P_\kappa(ds'\mid s,a_{\mathsf L}),
\label{eq:state-distribution-propagation}
\end{align}
\end{subequations}}where $\delta_{s_0}$ denotes the Dirac measure (point mass) at $s_0$.
Based on \eqref{eq:objective} and the definition of
$V_0^{\theta,\kappa}$,
\begin{equation}\label{eq:method-objective-orig}
    J(\theta;\kappa)
=
\mathbb E_{s_0\sim\rho_0}
\left[V_0^{\theta,\kappa}(s_0)\right].
\end{equation}
Since $\rho_0$ is independent of both $\theta$ and $\kappa$, $\nabla_\theta J(\theta;\kappa)=\mathbb E_{s_0\sim\rho_0}\bigl[\nabla_\theta V_0^{\theta,\kappa}(s_0)\bigr]$. 

Unrolling \eqref{eq:method-value-derivative} for $\mathcal T$ stages using \eqref{eq:state-distribution-propagation}, and then differentiating with respect to $\theta$ and $\kappa$, gives the following second-derivative decompositions. Here, $q$ indexes stages.
\begin{lemma}\label{lem:finite-depth}
Under Assumption~\ref{as:info-available}, for every $\mathcal T\in\{1,\dots,H\}$,

\vspace{-1em}
\begin{subequations}\label{eq:T-step2}
{\small\begin{align}\label{eq:T-step-hessian}
\nabla_{\theta\theta}^2 J(\theta;\kappa)
=&
\mathbb E_{s_0\sim\rho_0}
\Bigg[
\sum_{q=0}^{\mathcal T-1}\gamma^q
\frac{\partial}{\partial\theta^\top}
\left(
\int_{\mathcal S}
\mu_q^{\theta,\kappa}(ds\mid s_0)
G_q^{\theta,\kappa}(s)
\right)
\nonumber\\
&+
\gamma^{\mathcal T}
\int_{\mathcal S}
\nabla_\theta V_{\mathcal T}^{\theta,\kappa}(s)
\bigl[
\nabla_\theta
\mu_{\mathcal T}^{\theta,\kappa}(ds\mid s_0)
\bigr]^\top
\nonumber\\
&+
\gamma^{\mathcal T}
\int_{\mathcal S}
\mu_{\mathcal T}^{\theta,\kappa}(ds\mid s_0)
\nabla_{\theta\theta}^2
V_{\mathcal T}^{\theta,\kappa}(s)
\Bigg],\\
\label{eq:T-step-mixed}
\nabla_{\theta\kappa}^2 J(\theta;\kappa)
=&
\mathbb E_{s_0\sim\rho_0}
\Bigg[
\sum_{q=0}^{\mathcal T-1}\gamma^q
\frac{\partial}{\partial\kappa^\top}
\left(
\int_{\mathcal S}
\mu_q^{\theta,\kappa}(ds\mid s_0)
G_q^{\theta,\kappa}(s)
\right)
\nonumber\\
&+
\gamma^{\mathcal T}
\int_{\mathcal S}
\nabla_\theta V_{\mathcal T}^{\theta,\kappa}(s)
\bigl[
\nabla_\kappa
\mu_{\mathcal T}^{\theta,\kappa}(ds\mid s_0)
\bigr]^\top
\nonumber\\
&+
\gamma^{\mathcal T}
\int_{\mathcal S}
\mu_{\mathcal T}^{\theta,\kappa}(ds\mid s_0)
\nabla_{\theta\kappa}^2
V_{\mathcal T}^{\theta,\kappa}(s)
\Bigg].
\end{align}}
\end{subequations}
\end{lemma}
The derivatives inside the sums act on both $\mu_q^{\theta,\kappa}$ and $G_q^{\theta,\kappa}$. The proof of Lemma~\ref{lem:finite-depth} is given in the Appendix.

For any $\mathcal T\in\{1,\ldots,H\}$, evaluating \eqref{eq:T-step-hessian} and \eqref{eq:T-step-mixed} at $(\theta^\star(\kappa_{\mathrm{old}}),\kappa_{\mathrm{old}})$ gives $\mathbf H_{\mathrm{old}}$ and $\mathbf M_{\mathrm{old}}$, respectively. 
Thus, choosing the propagation depth $\mathcal T$ in Lemma~\ref{lem:finite-depth} only specifies where the exact decomposition is split; it does not introduce an approximation.
For $\mathcal T<H$, the remaining-horizon dependence appears through both the boundary terms and $G_q^{\theta,\kappa}$ through $Q_q^{\theta,\kappa}$.
To limit model-based propagation to $\mathcal T$ stages, in the following, we replace the value function $V_{\mathcal T}^{\theta,\kappa}$ by a fixed estimate, and recompute the preceding local contributions through Bellman recursion.

\subsection{Finite-Depth Sensitivity Approximation}
\label{sub:finite-depth-approximation}
From the reference data, take $N$ initial-state samples
$s_{i,0}\sim\rho_0$, $i=1,\ldots,N$, and choose a propagation
depth $\mathcal T\in\{1,\ldots,H\}$.
Let $\widehat V_{\mathcal T}$ denote an estimate of
$V_{\mathcal T}^{\theta^\star(\kappa_{\mathrm{old}}),
\kappa_{\mathrm{old}}}$. This estimate may be supplied by a reference-policy critic or fitted from reference trajectories.
Initialize $\widehat V_{\mathcal T}^{(\mathcal T),\theta,\kappa} :=\widehat V_{\mathcal T}$, holding this function fixed with respect to $\theta$ and $\kappa$. Set $\widehat V_{\mathcal T}=0$ at terminal states, and use $\widehat V_H\equiv0$ when $\mathcal T=H$. For $t=\mathcal T-1,\ldots,0$, define

\vspace{-1em}
{\small
\begin{subequations}
\label{eq:finite-depth-bellman}
\begin{align}
\widehat Q_t^{(\mathcal T),\theta,\kappa}
(s,a_{\mathsf L})
\nonumber
&:=
\int_{\mathcal S}
P_\kappa(ds'\mid s,a_{\mathsf L})\\
&\quad \cdot \Bigl[
r(s,a_{\mathsf L},s')
+
\gamma
\widehat V_{t+1}^{(\mathcal T),\theta,\kappa}(s')
\Bigr],
\label{eq:finite-depth-Q}
\\
\widehat V_t^{(\mathcal T),\theta,\kappa}(s)
&:=
\sum_{a_{\mathsf L}\in\mathcal A_{\mathsf L}}
\pi_\theta(a_{\mathsf L}\mid s)
\widehat Q_t^{(\mathcal T),\theta,\kappa}
(s,a_{\mathsf L}).
\label{eq:finite-depth-V}
\end{align}
\end{subequations}
}
Define the approximate local contribution by

{\small
\begin{equation}
\label{eq:finite-depth-local-gradient}
\widehat G_t^{(\mathcal T),\theta,\kappa}(s)
:=
\sum_{a_{\mathsf L}\in\mathcal A_{\mathsf L}}
\nabla_\theta\pi_\theta(a_{\mathsf L}\mid s)
\widehat Q_t^{(\mathcal T),\theta,\kappa}
(s,a_{\mathsf L}).
\end{equation}
}
The backward derivative recursions in the Appendix apply with $V_t^{\theta,\kappa}$, $Q_t^{\theta,\kappa}$, and $G_t^{\theta,\kappa}$ replaced by their corresponding truncated quantities. All parameter derivatives of $\widehat V_{\mathcal T}$ vanish because $\widehat V_{\mathcal T}$ is fixed. The distributions $\mu_t^{\theta,\kappa}$ and their derivatives are propagated forward from $\mu_0=\delta_{s_{i,0}}$ and zero initial derivatives.

Applying the derivation of Lemma~\ref{lem:finite-depth} to the truncated recursions eliminates its boundary derivative terms. Expanding the derivatives of the stage integrals by the product rule and replacing the expectation under $\rho_0$ by the sample average gives

\vspace{-1.5em}
{\small
\begin{subequations}
\label{eq:method-estimate}
\begin{align}
\label{eq:method-hessian-estimate}
\widehat{\mathbf H}_{\mathrm{old}}^{(\mathcal T)}
:=&
\frac{1}{N}\sum_{i=1}^{N}
\sum_{q=0}^{\mathcal T-1}\gamma^q
\Bigg[
\int_{\mathcal S}
\widehat G_q^{(\mathcal T),\theta,\kappa}(s)
\bigl[
\nabla_\theta
\mu_q^{\theta,\kappa}(ds\mid s_{i,0})
\bigr]^\top
\nonumber\\
&+
\int_{\mathcal S}
\mu_q^{\theta,\kappa}(ds\mid s_{i,0})
\frac{\partial
\widehat G_q^{(\mathcal T),\theta,\kappa}(s)}
{\partial\theta^\top}
\Bigg],
\\
\label{eq:method-mixed-estimate}
\widehat{\mathbf M}_{\mathrm{old}}^{(\mathcal T)}
:=&
\frac{1}{N}\sum_{i=1}^{N}
\sum_{q=0}^{\mathcal T-1}\gamma^q
\Bigg[
\int_{\mathcal S}
\widehat G_q^{(\mathcal T),\theta,\kappa}(s)
\bigl[
\nabla_\kappa
\mu_q^{\theta,\kappa}(ds\mid s_{i,0})
\bigr]^\top
\nonumber\\
&+
\int_{\mathcal S}
\mu_q^{\theta,\kappa}(ds\mid s_{i,0})
\frac{\partial
\widehat G_q^{(\mathcal T),\theta,\kappa}(s)}
{\partial\kappa^\top}
\Bigg],
\end{align}
\end{subequations}
}where all quantities on the right-hand sides are evaluated at $(\theta,\kappa)=(\theta^\star(\kappa_{\mathrm{old}}),\kappa_{\mathrm{old}})$. 
The approximation retains information about future return through $\widehat V_{\mathcal T}$ while setting its parameter derivatives to zero.
Furthermore, at $\mathcal T=H$, the boundary value is zero, so the construction retains full-horizon derivative propagation and introduces
no truncation.

Using the relation
$\mathbf H_{\mathrm{old}}\mathbf S_{\mathrm{old}}
+\mathbf M_{\mathrm{old}}=0$,
when $\widehat{\mathbf H}_{\mathrm{old}}^{(\mathcal T)}$ is nonsingular, the sensitivity estimate is given by
{\small
\begin{equation}
\label{eq:method-sensitivity-estimate}
\widehat{\mathbf S}_{\mathrm{old}}^{(\mathcal T)}
:=
-\left[\widehat{\mathbf H}_{\mathrm{old}}^{(\mathcal T)}
\right]^{-1}\widehat{\mathbf M}_{\mathrm{old}}^{(\mathcal T)}.
\end{equation}
}

For the parameter change
$\kappa_{\mathrm{new}}=\kappa_{\mathrm{old}}+\Delta\kappa$
in Section~\ref{sub:application}, the estimated policy sensitivity gives the initial policy parameter for adaptation in the target environment $\mathcal M(\kappa_{\mathrm{new}})$:
{\small
\begin{equation}
\label{eq:method-initialization}
\theta_{\mathrm{init}}
:=
\theta^\star(\kappa_{\mathrm{old}})
+
\widehat{\mathbf S}_{\mathrm{old}}^{(\mathcal T)}
\Delta\kappa.
\end{equation}
}

\section{Sensitivity Error Analysis}\label{sec:sensitivity-error}
This section analyzes the approximation error introduced by finite-depth derivative truncation. To isolate this structural truncation error from finite-sample and value-function approximation errors, we consider an ideal setting in which 
$\widehat V_{\mathcal T}
=V_{\mathcal T}^{\theta^\star(\kappa_{\mathrm{old}}),
\kappa_{\mathrm{old}}}$
and the sample averages in \eqref{eq:method-estimate} are replaced by exact expectations under $\rho_0$.
We perform this isolation because these two sources of error are complementary to the finite-depth truncation considered in this paper. %
Nevertheless, their practical effects are examined empirically in Section~V.

\subsection{Derivative Truncation Error}
Under the ideal setting described above, $\widehat{\mathbf H}_{\mathrm{old}}^{(\mathcal T)}$ and $\widehat{\mathbf M}_{\mathrm{old}}^{(\mathcal T)}$ differ from $\mathbf H_{\mathrm{old}}$ and $\mathbf M_{\mathrm{old}}$ only due to
the truncation of parameter-derivative propagation beyond propagation depth $\mathcal T$. Next, we analyze the bounds of these truncation errors under the following assumption.
\begin{assumption}[Uniform Derivative Bounds]
\label{ass:uniform-derivative-bounds}
There exist finite nonnegative constants
$R$, $C_\theta$, $C_{\theta\theta}$, and $C_\kappa$
such that

{\small
\begin{equation}
\label{eq:uniform-derivative-bounds}
\begin{aligned}
|r(s,a_{\mathsf L},s')|
&\le R,
\\
\sup_{s\in\mathcal S}
\sum_{a_{\mathsf L}\in\mathcal A_{\mathsf L}}
\left\|
\nabla_\theta\pi_\theta(a_{\mathsf L}\mid s)
\right\|_2
&\le C_\theta,
\\
\sup_{s\in\mathcal S}
\sum_{a_{\mathsf L}\in\mathcal A_{\mathsf L}}
\left\|
\nabla_{\theta\theta}^2
\pi_\theta(a_{\mathsf L}\mid s)
\right\|_2
&\le C_{\theta\theta},
\\
\sup_{w\in\mathbb R^{d_\kappa},~\|w\|_2=1}
\int_{\mathcal S}
\left|
w^\top\nabla_\kappa
P_\kappa(ds'\mid s,a_{\mathsf L})
\right|
&\le C_\kappa.
\end{aligned}
\end{equation}
}hold uniformly over states and actions. All derivatives are evaluated at $(\theta^\star(\kappa_{\mathrm{old}}),\kappa_{\mathrm{old}})$. The last integral is the total variation of the scalar signed measure $w^\top\nabla_\kappa P_\kappa(\cdot\mid s,a_{\mathsf L})$.%
\end{assumption}

Assumption~\ref{ass:uniform-derivative-bounds} provides uniform bounds on the reward and the policy and transition derivatives needed to bound each omitted derivative.
In particular, the total-variation bound on $\nabla_\kappa P_\kappa$ controls the effect of perturbations in the interactive agent's behavioral parameter on the transition distribution.

\begin{theorem}[Derivative Truncation Error]
\label{the:derivative-truncation}
Under Assumptions~\ref{as:info-available}-\ref{ass:uniform-derivative-bounds} and the ideal setting described in this section, 
for every $\mathcal T\in\{1,\ldots,H\}$, there hold

\vspace{-0.8em}
{\small
\begin{equation}
\label{eq:derivative-truncation-bounds}
\|\widehat{\mathbf H}_{\mathrm{old}}^{(\mathcal T)}
-\mathbf H_{\mathrm{old}}\|_2
\le \eta_{\mathbf H}^{(\mathcal T)},\quad
\|\widehat{\mathbf M}_{\mathrm{old}}^{(\mathcal T)}
-\mathbf M_{\mathrm{old}}\|_2
\le \eta_{\mathbf M}^{(\mathcal T)},
\end{equation}
}
where
{\small
\begin{align}
\label{eq:hessian-truncation-bound}
\eta_{\mathbf H}^{(\mathcal T)}
:={}R\gamma^{\mathcal T}
\sum_{q=0}^{H-\mathcal T-1}\gamma^q(q+1)
\left[C_{\theta\theta}
+(q+2\mathcal T)C_\theta^2\right],\\
\label{eq:mixed-truncation-bound}
\eta_{\mathbf M}^{(\mathcal T)}
:={}RC_\theta C_\kappa\gamma^{\mathcal T}
\sum_{q=0}^{H-\mathcal T-1}
\gamma^q(q+1)(q+2\mathcal T+1).
\end{align}
}An empty sum is zero.
Both bounds are nonincreasing in $\mathcal T$ and vanish at $\mathcal T=H$. For $0<\gamma<1$, %
the bounds additionally admit the following discounted envelopes:

\vspace{-0.5em}
{\small\begin{equation}
\label{eq:geometric-truncation-bounds}
\eta_{\mathbf H}^{(\mathcal T)}
\le
\frac{R[C_{\theta\theta}+2(H-1)C_\theta^2]}
{(1-\gamma)^2}\gamma^{\mathcal T},~
\eta_{\mathbf M}^{(\mathcal T)}
\le
\frac{R(2H-1)C_\theta C_\kappa}
{(1-\gamma)^2}\gamma^{\mathcal T}.
\end{equation}}
\end{theorem}

\begin{proof}
The proof is done by comparing the finite-depth estimates \eqref{eq:method-estimate} with the exact decompositions \eqref{eq:T-step2}.

Under the ideal setting of this section, $\widehat V_{\mathcal T}=V_{\mathcal T}^{\theta^\star(\kappa_{\mathrm{old}}), \kappa_{\mathrm{old}}}$. 
Thus, for $(\theta,\kappa)=(\theta^\star(\kappa_{\mathrm{old}}), \kappa_{\mathrm{old}})$, the backward recursions in \eqref{eq:finite-depth-bellman} yield $\widehat Q_t^{(\mathcal T),\theta,\kappa} =Q_t^{\theta,\kappa}$ and $\widehat G_t^{(\mathcal T),\theta,\kappa}
=G_t^{\theta,\kappa}$ for $t<\mathcal T$.
However, although their values are the same, their parameter derivatives are generally different, because $\widehat V_{\mathcal T}$ is held fixed with respect to $\theta$ and $\kappa$, whereas
$V_{\mathcal T}^{\theta,\kappa}$ depends on both parameters.
Consequently, when comparing \eqref{eq:method-estimate} with the exact decompositions \eqref{eq:T-step2} in Lemma~\ref{lem:finite-depth}, the terms multiplying derivatives of $\mu_t^{\theta,\kappa}$ cancel. The remaining error arises from setting the parameter derivatives of $V_{\mathcal T}^{\theta,\kappa}$ to zero in the finite-depth recursion and the resulting missing derivative contributions at earlier stages.

To characterize these omitted derivative contributions, unroll the
truncated Bellman recursion in \eqref{eq:finite-depth-bellman}.
Under the ideal setting, the estimates in
\eqref{eq:method-estimate} are equivalently the Hessian and mixed
derivative, evaluated at the reference parameters, of
\begin{equation}
\mathbb E_{\xi\sim\mathbb P_{\theta,\kappa}}
\left[
\sum_{t=0}^{\mathcal T-1}\gamma^t r_t
+\gamma^{\mathcal T}
V_{\mathcal T}^{
\theta^\star(\kappa_{\mathrm{old}}),
\kappa_{\mathrm{old}}}(s_{\mathcal T})
\right],
\end{equation}
where the boundary value
$V_{\mathcal T}^{\theta^\star(\kappa_{\mathrm{old}}),
\kappa_{\mathrm{old}}}$ is treated as fixed with respect to
$\theta$ and $\kappa$.
Thus, the finite-depth approximation retains parameter derivatives
through stages $0,\ldots,\mathcal T-1$ while omitting derivative terms involving
policy or transition factors at stage $\mathcal T$ or later.

We now decompose the derivative error according to the individual
stage-reward contributions in \eqref{eq:objective}. 
For a fixed
$q$, the expected contribution of $r_q$ can be written as

\vspace{-1em}
{\small
\begin{align}
\mathbb E_{\xi\sim\mathbb P_{\theta,\kappa}}[r_q]
=&
\int_{\mathcal S}\rho_0(ds_0)
\int_{\mathcal S}\mu_q^{\theta,\kappa}(ds_q\mid s_0)
\sum_{a_{\mathsf L,q}\in\mathcal A_{\mathsf L}}
\pi_\theta(a_{\mathsf L,q}\mid s_q)\nonumber\\
&\cdot
\int_{\mathcal S}
P_\kappa(ds_{q+1}\mid s_q,a_{\mathsf L,q})
r(s_q,a_{\mathsf L,q},s_{q+1}).
\end{align}}Expanding $\mu_q^{\theta,\kappa}$ using \eqref{eq:state-distribution-propagation} yields $q+1$ policy probabilities and $q+1$ transition kernels, corresponding to stages $0,\dots,q$. Since the reward and $\rho_0$ have no explicit parameter dependence on $\theta$ or $\kappa$, all derivatives of this expectation are retained when $q<\mathcal T$. For $q\ge\mathcal T$, the finite-depth approximation retains only the derivative terms whose differentiated factors occur before stage $\mathcal T$.

For $q\ge\mathcal T$, consider first the Hessian contribution of $r_q$. Differentiating the $q+1$ policy factors twice with respect to $\theta$ produces two types of terms.
Differentiating one policy probability twice gives $q+1$ terms. Differentiating two policy probabilities at distinct stages once each gives two mutually transposed terms per pair of stages, totaling $q(q+1)$ terms. The estimate $\widehat{\mathbf H}_{\mathrm{old}}^{(\mathcal T)}$ retains $\mathcal T$ terms differentiating one policy probability twice and $\mathcal T(\mathcal T-1)$ terms differentiating two policy probabilities at distinct stages once each, with all differentiated stages before $\mathcal T$.

For the mixed derivative, one $\theta$-derivative acts on one of the $q+1$ policy factors and one $\kappa$-derivative acts on one of the $q+1$ transition kernels.
Those contribute $(q+1)^2$ terms to $\mathbf M_{\mathrm{old}}$: there are $q+1$ choices of policy probability and $q+1$ choices of transition kernel to differentiate. The estimate $\widehat{\mathbf M}_{\mathrm{old}}^{(\mathcal T)}$ retains the $\mathcal T^2$ terms with both differentiated stages before $\mathcal T$.

Assumption~\ref{ass:uniform-derivative-bounds} bounds the above three types of terms in spectral norm by $RC_{\theta\theta}$, $RC_\theta^2$, and $RC_\theta C_\kappa$, respectively. %
The remaining undifferentiated policy probabilities and transition kernels form probability factors and therefore do not enlarge these bounds under summation and integration.
Multiplying by the omitted-term counts and $\gamma^q$, then applying the triangle inequality, gives

\vspace{-1em}
\begin{small}
\begin{align*}
\left\|
\widehat{\mathbf H}_{\mathrm{old}}^{(\mathcal T)}
-\mathbf H_{\mathrm{old}}
\right\|_2
\le&
R\sum_{q=\mathcal T}^{H-1}\gamma^q
\Big[
(q+1-\mathcal T)C_{\theta\theta}\\
&+\bigl(q(q+1)-\mathcal T(\mathcal T-1)\bigr)
C_\theta^2
\Big],\\
\left\|
\widehat{\mathbf M}_{\mathrm{old}}^{(\mathcal T)}
-\mathbf M_{\mathrm{old}}
\right\|_2
\le&
RC_\theta C_\kappa
\sum_{q=\mathcal T}^{H-1}\gamma^q
\bigl[(q+1)^2-\mathcal T^2\bigr].
\end{align*} \end{small}Reindexing both sums with $q=0,\ldots,H-\mathcal T-1$ yields \eqref{eq:hessian-truncation-bound} and \eqref{eq:mixed-truncation-bound}. The two right-hand sides are nonincreasing in $\mathcal T$ because increasing $\mathcal T$ removes the $q=\mathcal T$ terms and decreases the omitted-term count for every remaining stage. At $\mathcal T=H$, both sums are empty.

For $0<\gamma<1$ and $\mathcal T<H$, we have
$q+2\mathcal T\le 2H-2$ and $q+2\mathcal T+1\le 2H-1$ for $q=0,\ldots,H-\mathcal T-1$. Moreover,
there holds $\sum_{q=0}^{H-\mathcal T-1}\gamma^q(q+1)
\le
\sum_{q=0}^{\infty}\gamma^q(q+1)
=
(1-\gamma)^{-2}.$
Applying these inequalities gives
\eqref{eq:geometric-truncation-bounds}.
\end{proof}

\smallskip

Under the same ideal setting, the derivative bounds in Theorem~\ref{the:derivative-truncation} yield a sensitivity error bound. Since Assumption~\ref{as:info-available} gives $\mathbf H_{\mathrm{old}}\prec0$, let $m:=\lambda_{\min}(-\mathbf H_{\mathrm{old}})>0$. The following result follows directly.
\begin{corollary}[Sensitivity Truncation Error]
\label{cor:sensitivity-estimation-error}
Under the conditions of Theorem~\ref{the:derivative-truncation}, if $\eta_{\mathbf H}^{(\mathcal T)}<m$, then $\widehat{\mathbf H}_{\mathrm{old}}^{(\mathcal T)}$ is nonsingular, and the solution of \eqref{eq:method-sensitivity-estimate} satisfies
\begin{equation}
\label{eq:empirical-sensitivity-error}
\left\|
\widehat{\mathbf S}_{\mathrm{old}}^{(\mathcal T)}
-\mathbf S_{\mathrm{old}}
\right\|_2
\le
\frac{
\eta_{\mathbf M}^{(\mathcal T)}
+
\eta_{\mathbf H}^{(\mathcal T)}
\|\mathbf S_{\mathrm{old}}\|_2
}{
m-\eta_{\mathbf H}^{(\mathcal T)}
}.
\end{equation}
\end{corollary}

\section{Simulation}
We consider a pursuit--evasion game that captures the agent-induced environment change studied in Section~II: the physical dynamics and task remain fixed, while the decision policy of an interactive agent changes. 
The learning agent $\mathsf L$ jointly controls two attackers, $A_1$ and $A_2$, while a defender $D$ serves as the interactive agent $\mathsf O$. 
At each stage, within a grid world, the attackers move first, producing the intermediate state $\tilde s_t$. Then, the defender selects its move according to a behavior-parameterized policy $\psi(\cdot\mid\tilde s_t;\kappa)$, after which the next state is generated. Thus, the game follows the transition structure in \eqref{eq:induced-transition}, with changes in $\kappa$ affecting the
learning agent only through the defender's policy.

The interactive agent's behavioral parameter $\kappa$ characterizes how strongly the
defender relies on its belief about information hidden from it. Specifically,
among the two attackers, one of them ($\phi\in\{1,2\}$) holds a key and seeks to reach the target row $y=0$. The identity of the key holder $\phi$ is known to the attackers but hidden from the defender. The defender maintains a belief $b_t$ that $A_1$ holds the key and updates this belief by
Bayes' rule based on the observed attacker motions. Its policy is a mixture of two fixed pretrained policies, a belief-blind policy $\psi_0$ and a belief-aware policy $\psi_1$. We let
$\psi(\cdot\mid \tilde s_t;\kappa)
=
(1-\kappa)\psi_0(\cdot\mid \tilde s_t)
+
\kappa\psi_1(\cdot\mid\tilde s_t),~
\kappa\in[0,1].$
Hence, $\kappa$ controls the defender's reliance on its belief, and changing $\kappa$ changes the environment experienced by the learning agent. This provides a concrete instance of the parameterized family $\mathcal M(\kappa)$ in Section~II.

The state
$s_t=(s_t^D,s_t^{A_1},s_t^{A_2},\phi,b_t)$
is Markov for fixed $\kappa$. An episode terminates when an attacker
is captured, the key holder reaches the target row, or the horizon is
reached. The terminal reward reflects the key holder's progress toward
the target and the effectiveness of diverting the defender.
Initially, $\phi$ is uniform on $\{1,2\}$, $b_0$ is uniform on
$[0.1,0.9]$, and the agent positions are sampled according to a fixed distribution.
We consider two instances of the game for different evaluation purposes. 
A small scale uses a $6\times6$ grid with $H=15$ and $\gamma =0.9$. %
Its smaller state space allows exact evaluation of the derivative and sensitivity quantities used as benchmarks, so that we could analyze the estimation accuracy as the propagation depth changes.
A larger scale uses a $20\times20$ grid with $H=60$ and $\gamma =0.99$ to evaluate the effect of sensitivity-based initialization on zero-shot return and subsequent policy adaptation. 
In both instances, the learning agent's policy is trained using PPO. The PPO critic provides $\widehat V_{\mathcal T}$ when an estimated value function is used. 

The above description emphasizes the alignment between the experimental setting and the problem formulation. Additional environment and implementation details are provided in the project
GitHub repository~\cite{project_github}.

\begin{figure}[t]
    \centering
    \includegraphics[width=\linewidth]{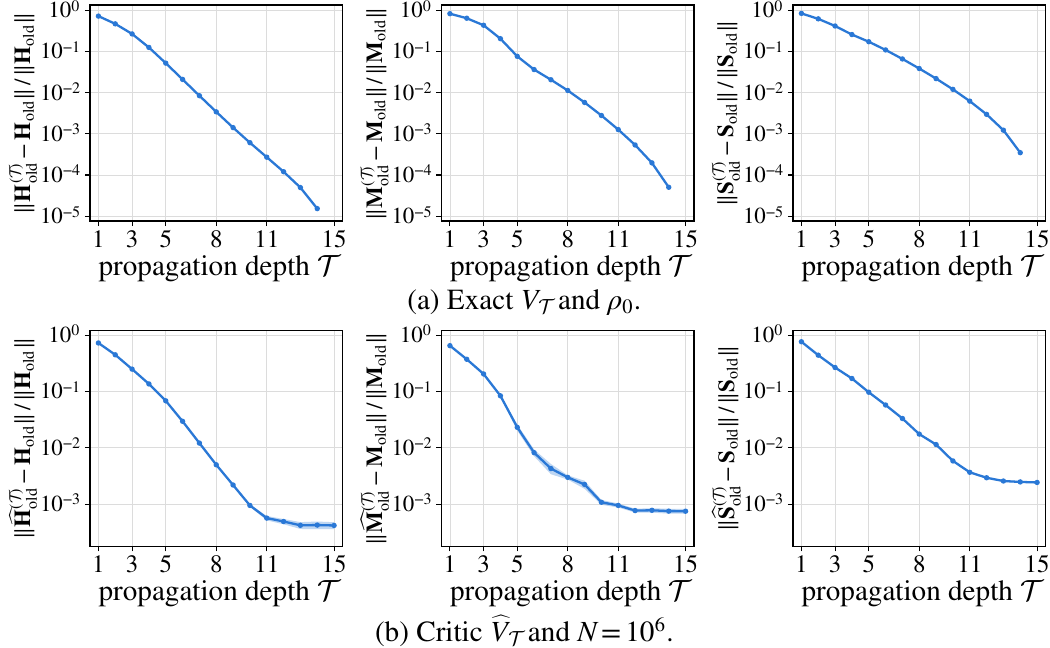}\vspace{-.5em}
    \caption{\small Errors of $\mathbf H$, $\mathbf M$, and $\mathbf S$ versus propagation depth $\mathcal T$.\vspace{-2em}}
    \label{fig:errors}
\end{figure}
\subsection{Sensitivity Estimation Accuracy}
\label{sub:sensitivity-accuracy}
We evaluate the estimates \eqref{eq:method-hessian-estimate}-\eqref{eq:method-mixed-estimate} on the small scale instance at $\kappa_{\mathrm{old}}=0.7$ and compute the sensitivity using \eqref{eq:method-sensitivity-estimate}. Relative errors are computed in the spectral norm, i.e., $\|\widehat{\mathbf H}_{\mathrm{old}}^{(\mathcal T)}-\mathbf H_{\mathrm{old}}\|_2/\|\mathbf H_{\mathrm{old}}\|_2$, and similarly for $\mathbf M_{\mathrm{old}}$ and $\mathbf S_{\mathrm{old}}$.

Fig.~\ref{fig:errors} shows the effect of propagation depth $\mathcal T$. 
In Fig.~\ref{fig:errors}(a), we use the exact boundary value
$V_{\mathcal T}^{\theta^\star(\kappa_{\mathrm{old}}),
\kappa_{\mathrm{old}}}$ and exact expectations under $\rho_0$. This aligns with the ideal setting in Sec.~IV that isolates the truncation error.
The observed reduction with increasing $\mathcal T$ is consistent with retaining more derivative terms, as Theorem~\ref{the:derivative-truncation} characterizes. %
The sensitivity error also decreases, which is consistent with Corollary~\ref{cor:sensitivity-estimation-error}, whose bound tightens as the derivative bounds decrease. %
Fig.~\ref{fig:errors}(b) evaluates the practical setting using
the critic estimates $\widehat V_{\mathcal T}$ and $N=10^6$ sampled initial states (without the ideal setting in Sec.~IV).
The errors decrease initially with $\mathcal T$ and then approach a
plateau near $10^{-3}$.
At intermediate depths $\mathcal T$, both value function approximation and sampling can contribute to the error. At $\mathcal T=H$, the truncation value is zero. The error comes from sampling or numerical computation.

Fig.~\ref{fig:comparison} compares the proposed method at $\mathcal T=3,5,7$ with a trajectory likelihood-ratio estimator~\cite{farquhar2019loaded}. The trajectory likelihood-ratio estimator uses policy likelihood derivatives to estimate $\mathbf H_{\mathrm{old}}$ and additionally uses the defender-policy score with respect to $\kappa$ to estimate $\mathbf M_{\mathrm{old}}$, with $N$ sampled reference trajectories per estimate. In contrast, the proposed method only samples $N$ initial states from $\rho_0$ and computes conditional expectations using the transition kernel and Bellman recursions. Thus, even at the same value of $N$, the proposed
method requires substantially less sampling data.

The proposed estimates are relatively insensitive to $N$ over the tested range and achieve low error even with a modest number of initial-state samples. This is because computing conditional expectations through the transition kernel avoids sampling subsequent actions and transitions.
As $N$ increases, the remaining error in our method is dominated by the finite-depth approximation and the fixed
boundary-value estimate, leading to a relatively flat curve whose level is mainly determined by the propagation depth. Increasing $\mathcal T$ lowers the entire curve by retaining more derivative propagation. Consequently, even at a small propagation depth, the proposed method can achieve competitive estimation accuracy without requiring the large numbers of complete trajectories needed by the trajectory likelihood-ratio estimator to reduce sampling error. At $\mathcal T=3$, the trajectory likelihood-ratio estimator eventually becomes more accurate as $N$ grows, whereas at $\mathcal T=5$ and $7$, the proposed method consistently achieves lower error.

\begin{figure}[t]
    \vspace*{-3.5mm}
    \centering
    \includegraphics[width=\linewidth]{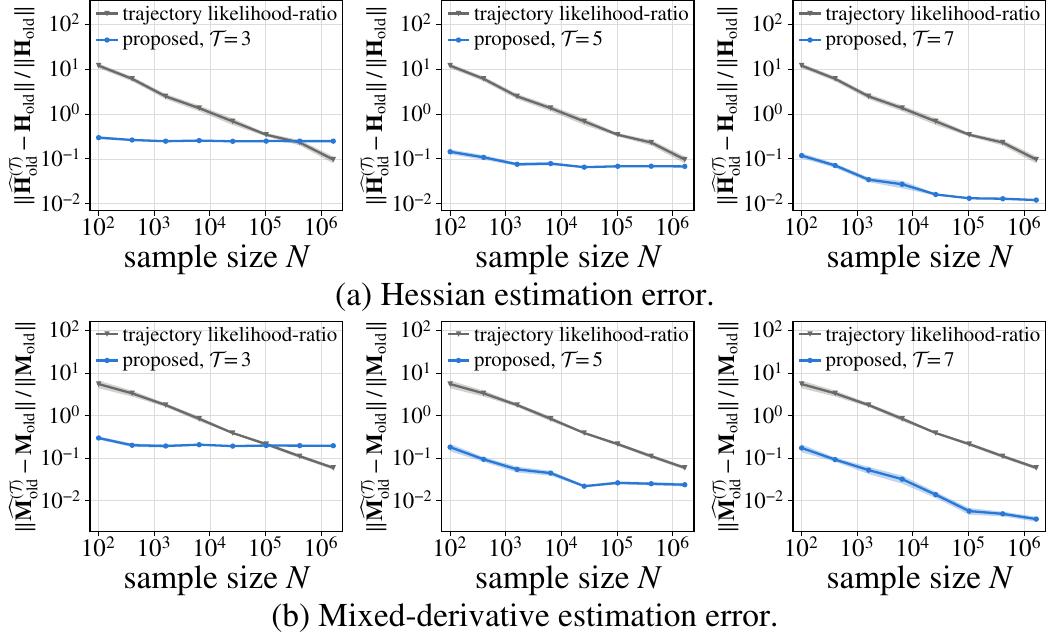}
    \caption{\small Relative estimation errors versus sample size $N$: the proposed method uses $N$ sampled initial states, whereas the trajectory likelihood-ratio estimator uses $N$ complete trajectories. Columns correspond to $\mathcal T=3,5,7$, from left to right. \vspace{-1.5em}}
    \label{fig:comparison}
\end{figure}

\subsection{Adaptation Performance}
\label{sub:adaptation-performance}
We next evaluate whether sensitivity-based initialization improves adaptation beyond directly reusing the reference policy. The large scale instance considers six transfers with $\kappa_{\mathrm{old}}\in\{0.5,0.6,0.7,1.0\}$ and $|\Delta\kappa|\in[0.2,0.6]$, covering both increases and decreases in the defender's behavioral parameter.  We compare sensitivity-based initialization with \emph{direct transfer}, initialized at $\theta^\star(\kappa_{\mathrm{old}})$, and \emph{training from scratch}. The two transfer methods use the same reference policy and identical fine-tuning settings. The only difference is their policy initialization. 

We estimate $\mathbf H_{\mathrm{old}}$ and $\mathbf M_{\mathrm{old}}$ using \eqref{eq:method-estimate} with $\mathcal T=1$ and $N=6000$ per estimate, and compute the corresponding sensitivity through \eqref{eq:method-sensitivity-estimate}. We average five independent sensitivity estimates and use the average in \eqref{eq:method-initialization}.
Each method is trained using PPO in the target environment for $3\times10^5$ steps with ten seeds. Performance is measured by the discounted return in \eqref{eq:objective}, evaluated over $2000$ episodes before training and every $2\times10^4$ training steps thereafter. The learning rate is $3\times10^{-4}$.

Fig.~\ref{fig:adaptation-performance} separates the benefit of reference policy reuse from that of the sensitivity-based initialization.
Both transfer methods perform better than training from scratch
within the fixed budget. 
More importantly, sensitivity-based initialization improves zero-shot mean return over direct transfer by $0.21$-$0.66$ across the six transfers. %
This advantage persists during fine-tuning in several cases and
gradually narrows in others as both methods adapt to the target
environment.
These results show that sensitivity-based initialization with propagation depth $\mathcal T=1$ can provide useful initializations for the tested transfers.
\begin{figure}[t]
    \vspace{-4mm}
    \centering
    \includegraphics[width=.9\linewidth]{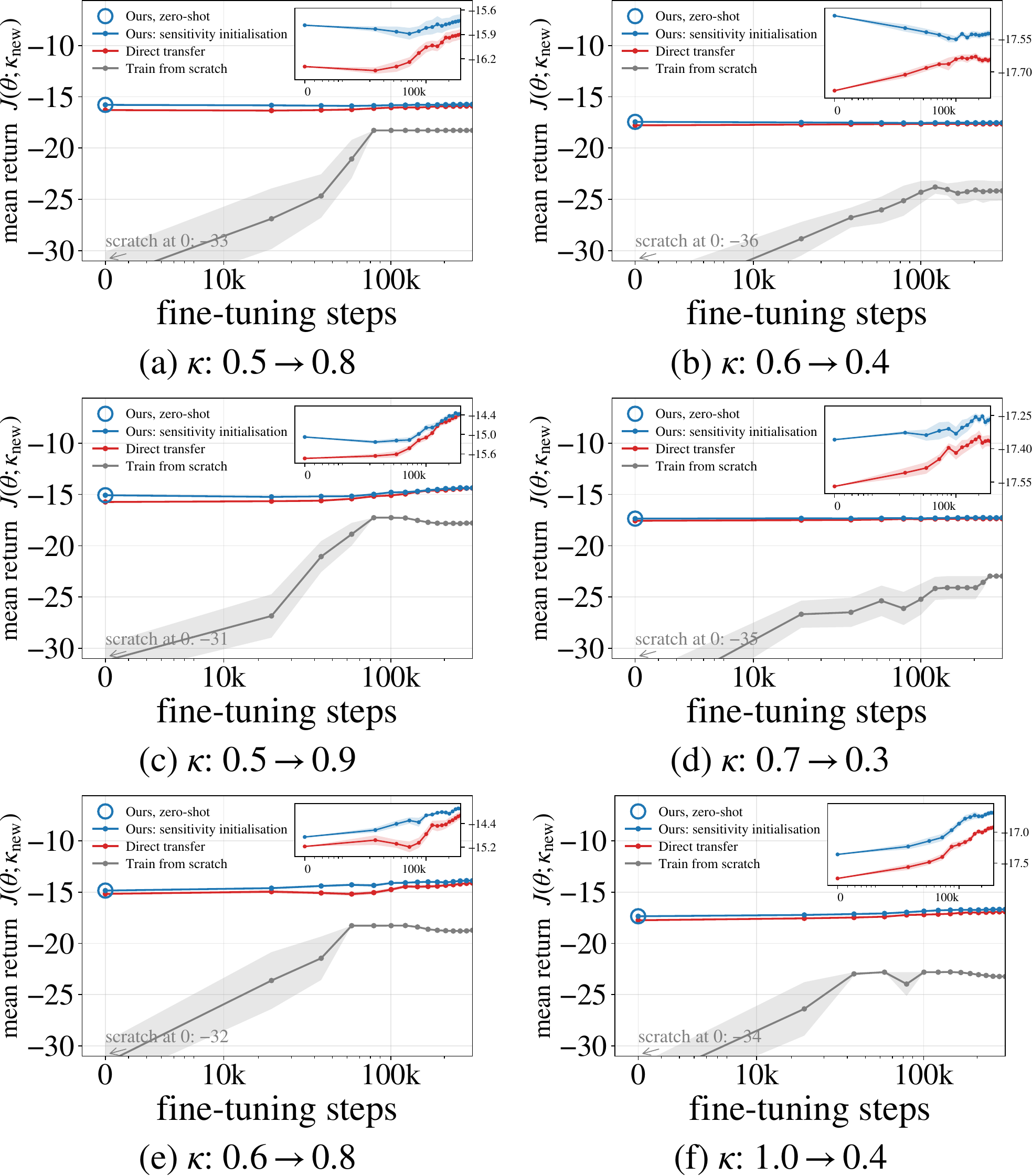}

    \caption{\small Adaptation performance for six changes in the defender's behavioral parameter $\kappa$. The horizontal axis reports target-environment training steps, and the vertical axis reports mean discounted return. %
    Insets compare the two transfer methods on an expanded vertical scale.\vspace{-1.5em}}
    \label{fig:adaptation-performance}
\end{figure}

\section{Conclusion}
This paper developed a finite-depth framework for estimating policy sensitivity to changes in an interactive agent's behavior. The adjustable propagation depth controls the tradeoff between model-based derivative
propagation and approximation accuracy. We theoretically prove that the resulting truncation-error bounds decrease with propagation depth, and validate the proposed method through pursuit--evasion experiments.
\textit{Limitations:} The current analysis assumes access to the reference transition kernel and its derivative with respect to $\kappa$. Moreover, the present method focuses on policy adaptation and truncates value-function derivative propagation by treating the boundary value estimate as parameter-independent.
Future work will relax the model-information requirements, incorporate sampling and value-function approximation errors into the analysis, and estimate the boundary value derivatives to enable joint policy and critic adaptation.

\appendix
\subsection*{Proof of Lemma~\ref{lem:finite-depth}}
\label{app:finite-depth-decomposition}

Since $\rho_0$ is parameter-independent,
$\nabla_\theta J
=\mathbb E_{s_0\sim\rho_0}
[\nabla_\theta V_0^{\theta,\kappa}(s_0)]$.
Unrolling \eqref{eq:method-value-derivative}
from stage $0$ for $\mathcal T$ stages and using
\eqref{eq:state-distribution-propagation} gives

\vspace{-1em}
\begin{small}
\begin{equation}
\label{eq:T-step-gradient}
\begin{aligned}
\nabla_\theta J(\theta;\kappa)
=&
\mathbb E_{s_0\sim\rho_0}
\Bigg[
\sum_{q=0}^{\mathcal T-1}\gamma^q
\int_{\mathcal S}
\mu_q^{\theta,\kappa}(ds\mid s_0)
G_q^{\theta,\kappa}(s)
\\
&+
\gamma^{\mathcal T}
\int_{\mathcal S}
\mu_{\mathcal T}^{\theta,\kappa}(ds\mid s_0)
\nabla_\theta V_{\mathcal T}^{\theta,\kappa}(s)
\Bigg].
\end{aligned}
\end{equation}
\end{small}Differentiating \eqref{eq:T-step-gradient} with respect to $\theta$ and $\kappa$ gives \eqref{eq:T-step-hessian} and \eqref{eq:T-step-mixed}, respectively. The product rule applied to the boundary integral produces one term from the state distribution and one from the value gradient. 
The interchange of differentiation and integration follows from the
regularity condition stated in Section~III-A.

For completeness, the derivative recursions are given below.
For fixed $s_0$,

\vspace{-1em}
{\small
\begin{equation*}
\begin{aligned}
\nabla_\theta&
\mu_{t+1}^{\theta,\kappa}(ds'\!\mid\! s_0)
=\!\int_{\mathcal S}
\nabla_\theta\mu_t^{\theta,\kappa}(ds\!\mid\! s_0)
\!\!\sum_{a_{\mathsf L}\in\mathcal A_{\mathsf L}}
\!\pi_\theta(a_{\mathsf L}\!\mid\! s)
P_\kappa(ds'\!\mid\! s,a_{\mathsf L})
\\
&+
\int_{\mathcal S}
\mu_t^{\theta,\kappa}(ds\mid s_0)
\sum_{a_{\mathsf L}\in\mathcal A_{\mathsf L}}
\nabla_\theta\pi_\theta(a_{\mathsf L}\mid s)
P_\kappa(ds'\mid s,a_{\mathsf L}).
\end{aligned}
\end{equation*}
}

\noindent Differentiating \eqref{eq:state-distribution-propagation} with respect to $\kappa$ gives
\begin{small}
\begin{equation*}
\begin{aligned}
\nabla_\kappa&
\mu_{t+1}^{\theta,\kappa}(ds'\!\mid\! s_0)
=\!\int_{\mathcal S}
\nabla_\kappa\mu_t^{\theta,\kappa}(ds\!\mid\! s_0)
\!\!\sum_{a_{\mathsf L}\in\mathcal A_{\mathsf L}}
\!\pi_\theta(a_{\mathsf L}\!\mid\! s)
P_\kappa(ds'\!\mid\! s,a_{\mathsf L})
\\
&+
\int_{\mathcal S}
\mu_t^{\theta,\kappa}(ds\mid s_0)
\sum_{a_{\mathsf L}\in\mathcal A_{\mathsf L}}
\pi_\theta(a_{\mathsf L}\mid s)
\nabla_\kappa P_\kappa(ds'\mid s,a_{\mathsf L}).
\end{aligned}
\end{equation*}
\end{small}
\noindent Differentiating \eqref{eq:method-local-gradient} gives
\begin{small}
\begin{equation*}
\begin{aligned}
\frac{\partial G_t^{\theta,\kappa}(s)}
{\partial\theta^\top}
&=
\sum_{a_{\mathsf L}\in\mathcal A_{\mathsf L}}
\Bigl[
Q_t^{\theta,\kappa}(s,a_{\mathsf L})
\nabla_{\theta\theta}^2
\pi_\theta(a_{\mathsf L}\mid s)
\\
&\quad+
\nabla_\theta\pi_\theta(a_{\mathsf L}\mid s)
\bigl[
\nabla_\theta Q_t^{\theta,\kappa}(s,a_{\mathsf L})
\bigr]^\top
\Bigr],
\\[2pt]
\frac{\partial G_t^{\theta,\kappa}(s)}
{\partial\kappa^\top}
&=
\sum_{a_{\mathsf L}\in\mathcal A_{\mathsf L}}
\nabla_\theta\pi_\theta(a_{\mathsf L}\mid s)
\bigl[
\nabla_\kappa Q_t^{\theta,\kappa}(s,a_{\mathsf L})
\bigr]^\top.
\end{aligned}
\end{equation*}
\end{small}
Differentiating the Bellman recursions yields
\begin{small}
\begin{align*}
\nabla_\kappa Q_t^{\theta,\kappa}(s,a_{\mathsf L})
=&
\int_{\mathcal S}
\nabla_\kappa P_\kappa(ds'\mid s,a_{\mathsf L})
\Bigl[
r(s,a_{\mathsf L},s')
+\gamma V_{t+1}^{\theta,\kappa}(s')
\Bigr]
\nonumber\\
&+
\gamma\int_{\mathcal S}
P_\kappa(ds'\mid s,a_{\mathsf L})
\nabla_\kappa V_{t+1}^{\theta,\kappa}(s'),
\\
\nabla_\kappa V_t^{\theta,\kappa}(s)
=&
\sum_{a_{\mathsf L}\in\mathcal A_{\mathsf L}}
\pi_\theta(a_{\mathsf L}\mid s)
\nabla_\kappa Q_t^{\theta,\kappa}(s,a_{\mathsf L}),
\\
\nabla_{\theta\theta}^2V_t^{\theta,\kappa}(s)
=&
\frac{\partial G_t^{\theta,\kappa}(s)}
{\partial\theta^\top}
+
\sum_{a_{\mathsf L}\in\mathcal A_{\mathsf L}}
\Bigl[
\pi_\theta(a_{\mathsf L}\mid s)
\nabla_{\theta\theta}^2
Q_t^{\theta,\kappa}(s,a_{\mathsf L})
\\
&+
\nabla_\theta Q_t^{\theta,\kappa}(s,a_{\mathsf L})
\bigl[
\nabla_\theta\pi_\theta(a_{\mathsf L}\mid s)
\bigr]^\top
\Bigr],
\\
\nabla_{\theta\kappa}^2V_t^{\theta,\kappa}(s)
=&
\frac{\partial G_t^{\theta,\kappa}(s)}
{\partial\kappa^\top}
+
\sum_{a_{\mathsf L}\in\mathcal A_{\mathsf L}}
\pi_\theta(a_{\mathsf L}\mid s)
\nabla_{\theta\kappa}^2
Q_t^{\theta,\kappa}(s,a_{\mathsf L}),
\\
\nabla_{\theta\theta}^2
Q_t^{\theta,\kappa}(s,a_{\mathsf L})
=&
\gamma\int_{\mathcal S}
P_\kappa(ds'\mid s,a_{\mathsf L})
\nabla_{\theta\theta}^2
V_{t+1}^{\theta,\kappa}(s'),
\\
\nabla_{\theta\kappa}^2
Q_t^{\theta,\kappa}(s,a_{\mathsf L})
=&
\gamma\int_{\mathcal S}
\nabla_\theta V_{t+1}^{\theta,\kappa}(s')
\bigl[
\nabla_\kappa P_\kappa(ds'\mid s,a_{\mathsf L})
\bigr]^\top
\nonumber\\
&+
\gamma\int_{\mathcal S}
P_\kappa(ds'\mid s,a_{\mathsf L})
\nabla_{\theta\kappa}^2
V_{t+1}^{\theta,\kappa}(s').
\end{align*}
\end{small}

Since $V_H^{\theta,\kappa}\equiv0$, %
all of its first- and second-order parameter derivatives are zero.
For the exact decompositions, the backward recursions start at stage $H$ and supply the values and derivatives required at stage $\mathcal T$ and earlier. Combined with the forward distribution recursions, they determine the quantities in \eqref{eq:T-step-hessian} and \eqref{eq:T-step-mixed}. For the truncated estimates, the backward recursions instead start at stage $\mathcal T$ with $\widehat V_{\mathcal T}$ and zero parameter
derivatives of the boundary value. %

\bibliographystyle{IEEEtran}
\bibliography{root}

\end{document}

%% file: math_commands.tex
\usepackage{amsmath,amsfonts,bm}

\newtheorem{assumption}{Assumption}

\newtheorem{theorem}{Theorem}
\newtheorem{corollary}{Corollary}
\newtheorem{remark}{Remark}

\def\eqref#1{(\ref{#1})}

\def\1{\bm{1}}

\DeclareMathAlphabet{\mathsfit}{\encodingdefault}{\sfdefault}{m}{sl}
\SetMathAlphabet{\mathsfit}{bold}{\encodingdefault}{\sfdefault}{bx}{n}

